\pdfoutput=1

\documentclass[10pt]{article} 
\usepackage[preprint]{tmlr}

\usepackage{amsmath,amsfonts,bm}

\def\eqref#1{equation~\ref{#1}}

\def\1{\bm{1}}

\DeclareMathAlphabet{\mathsfit}{\encodingdefault}{\sfdefault}{m}{sl}
\SetMathAlphabet{\mathsfit}{bold}{\encodingdefault}{\sfdefault}{bx}{n}

\DeclareMathOperator*{\argmin}{arg\,min}

\usepackage{hyperref}
\usepackage{url}
\usepackage{graphicx, color}

\usepackage{amsmath}
\usepackage{amssymb}
\usepackage{mathtools}
\usepackage{amsthm}
\usepackage{bbm}
\usepackage{booktabs}

\usepackage[linesnumbered,ruled,vlined]{algorithm2e}
\SetKwInput{KwInput}{Input}
\SetKwInput{KwOutput}{Output}

\theoremstyle{plain}
\newtheorem{theorem}{Theorem}[section]
\newtheorem{proposition}[theorem]{Proposition}

\theoremstyle{definition}

\theoremstyle{remark}

\usepackage{listings}
\title{Embarrassingly Simple Text Watermarks}

\author{\name Ryoma Sato \email r.sato@ml.ist.i.kyoto-u.ac.jp \\
      \addr Kyoto University\\
      Okinawa Institute of Science and Technology
      \AND
      \name Yuki Takezawa \email yuki-takezawa@ml.ist.i.kyoto-u.ac.jp \\
      \addr Kyoto University\\
      Okinawa Institute of Science and Technology
      \AND
      \name Han Bao \email bao@i.kyoto-u.ac.jp\\
      \addr Kyoto University\\
      Okinawa Institute of Science and Technology
      \AND
      \name Kenta Niwa \email kenta.niwa.bk@hco.ntt.co.jp \\
      \addr NTT Communication Science Laboratories
      \AND
      \name Makoto Yamada \email makoto.yamada@oist.jp\\
      \addr Okinawa Institute of Science and Technology}

\def\month{MM}  
\def\year{YYYY} 
\def\openreview{\url{https://openreview.net/forum?id=XXXX}} 

\begin{document}

\maketitle

\begin{abstract}
  We propose \textsc{Easymark}, a family of embarrassingly simple yet effective watermarks. Text watermarking is becoming increasingly important with the advent of Large Language Models (LLM). LLMs can generate texts that cannot be distinguished from human-written texts. This is a serious problem for the credibility of the text. \textsc{Easymark} is a simple yet effective solution to this problem. \textsc{Easymark} can inject a watermark without changing the meaning of the text at all while a validator can detect if a text was generated from a system that adopted \textsc{Easymark} or not with high credibility. \textsc{Easymark} is extremely easy to implement so that it only requires a few lines of code. \textsc{Easymark} does not require access to LLMs, so it can be implemented on the user-side when the LLM providers do not offer watermarked LLMs. In spite of its simplicity, it achieves higher detection accuracy and BLEU scores than the state-of-the-art text watermarking methods. We also prove the impossibility theorem of perfect watermarking, which is valuable in its own right. This theorem shows that no matter how sophisticated a watermark is, a malicious user could remove it from the text, which motivate us to use a simple watermark such as \textsc{Easymark}. We carry out experiments with LLM-generated texts and confirm that \textsc{Easymark} can be detected reliably without any degradation of BLEU and perplexity, and outperform state-of-the-art watermarks in terms of both quality and reliability. 
\end{abstract}

\section{Introduction}

With the advent of large language models (LLMs)~\citep{brown2020language,openai2023gpt4}, text watermarking is becoming increasingly important~\citep{kirchenbauer2023watermark,zhao2023provable,abdelnabi2021adversarial}. The quality of texts generated by LLMs is so high that it is difficult to distinguish them from human-written texts \citep{clark2021all,jakesch2023human,sadasivan2023can}. This increases the risk of abuse of automatically generated text. For example, a malicious user can generate a fake news article and spread it on social media. Some users may automatically generate large numbers of blog posts and try to earn advertising fees. Some students may use LLMs to generate essays and submit them to their teachers. In order to prevent such abuses, it is important to be able to detect automatically generated texts.

\cite{kirchenbauer2023watermark} proposed a method to have LLMs generate watermarked text. A basic idea is to split a vocabulary into red and green words. LLMs are forced to generate many green words. A validator can detect if a text was generated from a system that adopted this method or not by checking the ratio of green words. If a text contains too many green words to be generated by a human, the text is considered to be generated by an LLM. This method is effective in detecting automatically generated texts. \cite{takezawa2023necessary} elaborated this idea by precisely controlling the number of green words.

However, adding these watermarks harms the quality of the generated texts because the LLMs are forced to generate less diverse texts to include a sufficient number of green words. Specifically, it has been observed that the BLEU score and perplexity are degraded when the watermarks are added~\citep{takezawa2023necessary}. This fact hinders LLM vendors such as OpenAI and Microsoft from adopting text watermarks because the quality greatly affects user experience. Besides, incorporating these watermarks requires a lot of engineering effort, which makes practitioners all the more hesitate to adopt them. Worse, these watermarks require steering the LLM. Therefore, users cannot enjoy watermarked LLMs until the LLM providers adopt them. This is a serious problem because LLM providers may not adopt them for the above reasons.

To overcome these problems, we propose \textsc{Easymark}, a family of embarrassingly simple text watermarking methods that exploit the specifications of character codes. \textsc{Easymark} perfectly resolves the above concerns. First, \textsc{Easymark} does not degrade the quality of texts at all. The watermarked text looks the same as the original text. The degradation of BLEU and perplexity is literally zero. Second, \textsc{Easymark} is extremely easy to implement with only a few lines of code. \textsc{Easymark} is a plug-and-play module, and one does not need to modify the decoding program, while the methods proposed by \cite{kirchenbauer2023watermark} and \cite{takezawa2023necessary} need to modify the decoding algorithms. Therefore, \textsc{Easymark} can be implemented on the user's side. For example, \textsc{Easymark} can be installed on highschool computers as a browser add-on, and teachers can use watermarks even if LLM providers do not implement it. Note that \cite{hacker2023regulating} called for ``markings that are easy to use and recognize, but hard to remove by average users,'' \cite{grinbaum2022ethical} called for ``unintrusive, yet easily accessible marks of the machine origin'', and \textsc{Easymark} meets these requirements. These advantages make \textsc{Easymark} a practical method of text watermarking. \textsc{Easymark} is also a good starter before adopting more sophisticated methods. We also note that the approach of \textsc{Easymark} is orthogonal to the existing LLM watermarks, and can be combined with them to reinforce the reliability of the text watermarking.

We carry out experiments with LLM-generated texts and we confirm that \textsc{Easymark} can be detected reliably without any degradation of BLEU and perplexity.

You can try \textsc{Easymark} at \url{https://easymarkdemo.github.io/} in a few seconds. We encourage the readers to try it.

The contributions of this paper are as follows:

\begin{itemize}
    \item We propose \textsc{Easymark}, a family of embarrassingly simple text watermarking methods that exploit different Unicode codepoints that have the same meanings. \textsc{Easymark} does not degrade the quality of texts at all and is extremely easy to implement.
    \item \textsc{Easymark} includes variants for printed and CJK texts, which extend the applicability of \textsc{Easymark}.
    \item We prove the impossibility theorem of perfect watermarking, which is valuable in its own right. This theorem shows any reliable watermark, including an elaborated one, can be removed by a malcious user. This theorem motivates us to use simple watermarks like \textsc{Easymark} because we cannot avoid the vulnerability with even elaborated watermarks.
    \item We carry out experiments with LLM-generated texts and we confirm that \textsc{Easymark} can be detected reliably without any degradation of BLEU and perplexity.
\end{itemize}

\section{Problem Setting}

The goal is to design two functions, \texttt{add\_watermark} and \texttt{detect\_watermark}. \texttt{add\_watermark} takes a raw text as input and returns a watermarked text. \texttt{detect\_watermark} takes a text as input and returns a boolean value indicating whether the text was generated from a system that adopted \texttt{add\_watermark} or not. The requirements for these functions are (i) \texttt{add\_watermark} should not change the meaning of the text, and (ii) \texttt{detect\_watermark} should be able to detect the watermark with high credibility. LLM providers can incorporate our method as a plug-and-play module that is passed through after an LLM generates texts. We do not pose any assumptions on the input texts to be watermarked, while some existing methods can only add watermarks to LLM-generated texts \citep{kirchenbauer2023watermark,takezawa2023necessary}. Therefore, our problem setting is more general than the existing ones. For example, a human news writer can use our method to add watermarks to their articles.

\section{Proposed Method}

We propose a family of easy text watermarks, \textsc{Easymark}, which exploits the fact that Unicode has many codepoints with the same or similar appearances. \textsc{Easymark} has three variants, \textsc{Whitemark}, \textsc{Variantmark}, and \textsc{Printmark}, which are suitable for different scenarios. \textsc{Whitemark} is the easiest one and is suitable for digital texts. \textsc{Variantmark} is suitable for texts with Chinese characters. \textsc{Printmark} is suitable for printed texts. We will explain each variant in detail in the following subsections.

\subsection{Embarassingly Easy Watermark (Whitemark)}

\begin{figure}[tb]
  \centering
  \includegraphics[width=0.6\hsize]{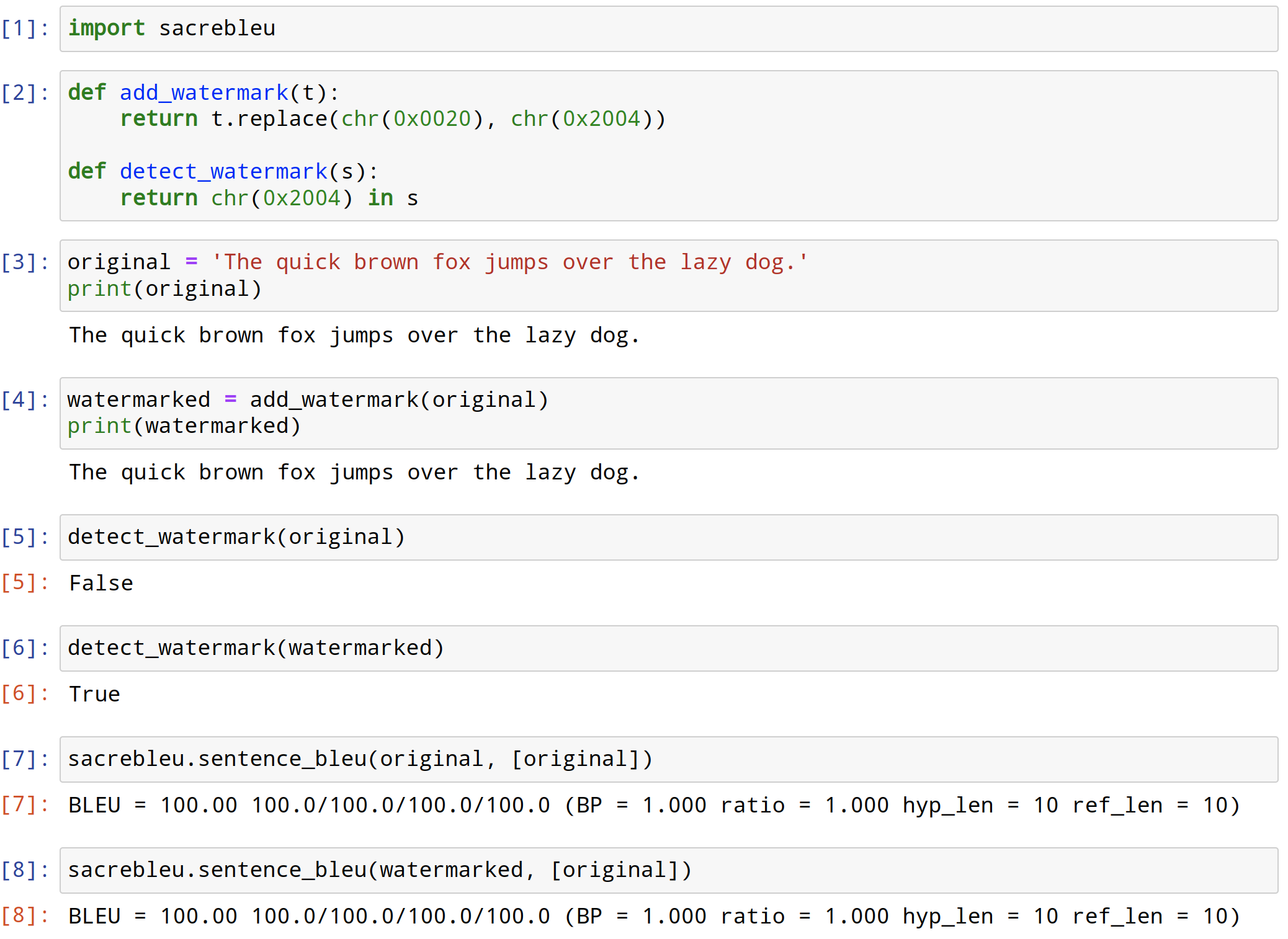}
  \caption{A screenshot of \textsc{Whitemark} on Jupyter noteook. \texttt{[2]}: The implementation of the \textsc{Whitemark} algorithm. \texttt{[3]} \texttt{[4]}: The original text and watermarked text look the same. \texttt{[5]} \texttt{[6]}: The watermarked text can be detected by the \texttt{detect\_watermark} function. \texttt{[7]} \texttt{[8]}: The Sacrebleu library identifies the original and watermarked texts.}
  \label{fig: execution}
\end{figure}

\begin{lstlisting}[caption={A function to add a watermark and detect it.}, label={lst: watermark}]
  def add_watermark(t):
      return t.replace(chr(0x0020), chr(0x2004))
  
  def detect_watermark(s):
      return chr(0x2004) in s
\end{lstlisting}

\textsc{Whitemark} is the simplest method that exploits the fact that Unicode has many codepoints for whitespace and replaces a whitespace (U+0020) with another codepoint of a whitespace, e.g., U+2004. The existence of \textsc{Whitemark} can be detected by counting the number of U+2004 in the texts. \textsc{Whitemark} does not change the meaning of a text at all. Listing \ref{lst: watermark} shows the Python code of \textsc{Whitemark}. An example of execution is shown in Figure \ref{fig: execution}. The appearance of a text does not change with \textsc{Whitemark}. The Sacrebleu library~\citep{post2018call} and SentencePiece library~\citep{kudo2018sentencepiece} identify the raw text and the text watermarked by \textsc{Whitemark}. These observations indicate \textsc{Whitemark} does not change the contents at all. Nevertheless, the \texttt{detect\_watermark} succeeds in detecting the watermark. Note that \textsc{Whitemark} does not disappear by electrical copy and paste.

\textsc{Whitemark} has the following preferable properties:

\begin{proposition} \label{prop: bleu}
  The BLEU scores and perplexity of the raw text and the text with \textsc{Whitemark} are the same.
\end{proposition}

\begin{proof}
  As \textsc{Whitemark} does not change any printable characters, this does not affect the BLEU score and perplexity.
\end{proof}

\begin{proposition} \label{prop: detect}
  If the original text has a whitespace \textup{(U+0020)} and does not have any \textup{U+2004}, the text with \textsc{Whitemark} can be detected with 100\% accuracy. More precisely, let $p_{\textup{nat}}(x)$ be the probability distribution of natural texts and let $p_{\textup{in}}(x)$ be the probability distribution of a text $x$ to be input to \textup{\textsc{Whitemark}}. Let \begin{align}
    \delta_{\textup{FP}} &\stackrel{\textup{def}}{=} 1 - \textup{Pr}_{x \sim p_{\textup{nat}}(x)}[x \textup{ does not contain \textup{U+2004}}], \\
    \delta_{\textup{FN}} &\stackrel{\textup{def}}{=} 1 - \textup{Pr}_{x \sim p_{\textup{in}}(x)}[x \textup{ contains a whitespace}].
  \end{align} Then, \begin{align}
    \textup{Pr}_{x \sim p_{\textup{nat}}(x)}[\texttt{\textup{detect\_watermark}}(x) = \texttt{\textup{False}}] &\ge 1 - \delta_{\textup{FP}}, \label{eq: detect-1} \\
    \textup{Pr}_{x \sim p_{\textup{in}}(x)}[\texttt{\textup{detect\_watermark}}(\texttt{\textup{add\_watermark}}(x)) = \texttt{\textup{True}}] &\ge 1 - \delta_{\textup{FN}}. \label{eq: detect-2}
  \end{align}
\end{proposition}

\begin{proof}
  Suppose $x$ does not contain U+2004. This event holds with probability $1 - \delta_{\text{FP}}$ under $p_{\text{nat}}(x)$. As $x$ does not contain U+2004, \texttt{detect\_watermark}$(x)$ returns \texttt{False} under this condition. Thus, Eq. \ref{eq: detect-1} holds.
  
  Suppose $x$ contains a whitespace. This event holds with probability $1 - \delta_{\text{FN}}$ under $p_{\text{in}}(x)$. Under this condition, \texttt{add\_watermark}$(x)$ contains U+2004 because $x$ contains a whitespace. Therefore, \texttt{detect\_watermark}(\texttt{add\_watermark}$(x))$ returns \texttt{True}. Thus, Eq. \ref{eq: detect-2} holds.
\end{proof}

As $\delta_{\text{FP}}$ and $\delta_{\text{FN}}$ are extremely low in real-world scenarios, Proposition \ref{prop: detect} indicates \textsc{Whitemark} can be detected almost perfectly.
We confirmed that $\delta_{\text{FP}} = 0$ held in the WMT-14 dataset. We also confirmed that $\delta_{\text{FN}} = 0$ held in the $100$ texts generated by GPT-3.5, NLLB-200 \citep{costajussa2022no}, and LLaMA \citep{touvron2023llama}. More experimental results can be found in Section \ref{sec: experiment}.

Besides, \textsc{Whitemark} is robust to manual edit. \textsc{Whitemark} can be detected even after the text is edited unless all of the whitespaces are edited, which usually happens when a user rewrites the entire text, in which case we do not need to claim the text is generated from an LLM. We will return to this point in Section \ref{sec: limitation}.

Another preferable property of \textsc{Whitemark} is that it can be implemented in a streaming manner. Many LLM applications require real-time generation of texts. For example, chatbots like ChatGPT respond to a user's message as soon as tokens are generated. Elaborated watermark algorithms such as the dynamic programming of NS-Watermark~\citep{takezawa2023necessary} cannot inject a watermark into a text in a streaming manner because they require the entire text to be generated before adding a watermark. By contrast, \textsc{Whitemark} is suitable for real-time applications.

Last but not least, \textsc{Whitemark} can be implemented on the user-side \citep{sato2022clear,sato2022towards} because it does not require access to the LLM while the existing methods require the LLM provider to implement them. The LLM provider may hesitate to adopt the existing methods because they require a lot of engineering effort and may degrade the quality of the texts. In that case, what users can do is to wait for the LLM provider to adopt the methods. By contrast, users can introduce \textsc{Whitemark} on their own. For example, highschool teachers can install \textsc{Whitemark} as a browser add-on and use it to detect automatically generated essays even if the LLM provider does not implement it.

\subsection{Steganography}

\begin{figure}[tb]
  \centering
  \includegraphics[width=\hsize]{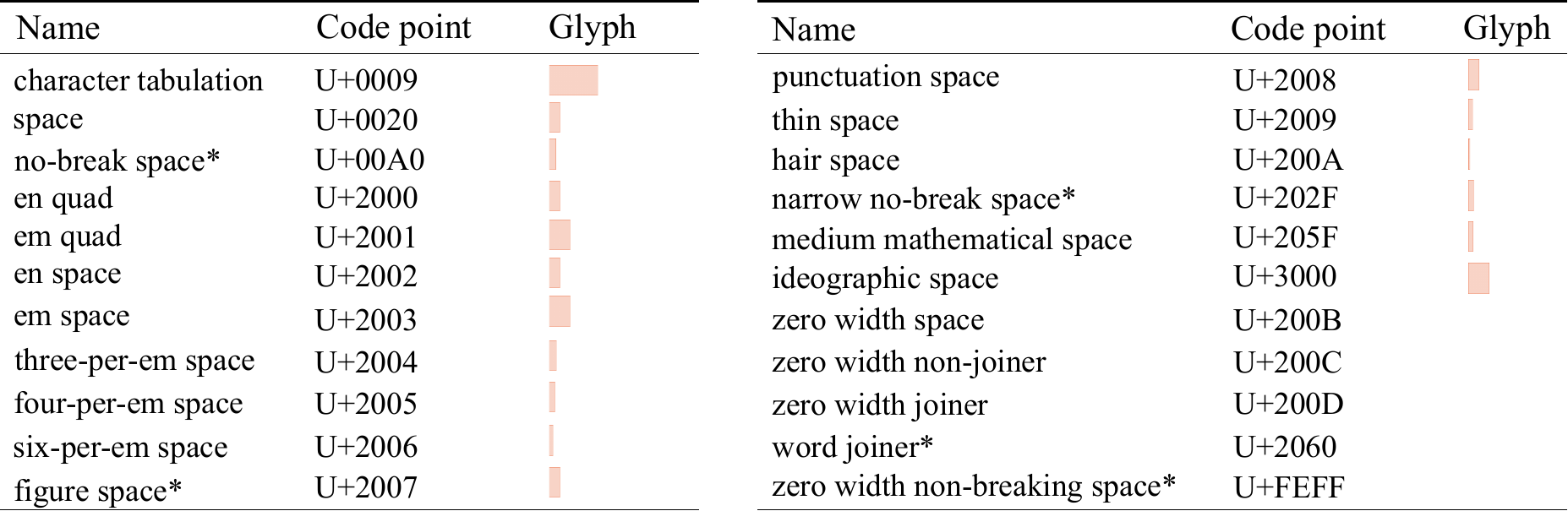}
  \caption{A list of whitespaces. The red bars indicate the lengths of the whitespaces. These codepoints can be used for steganography and watermarks. *: no-break space.}
  \label{fig: whitespaces}
\end{figure}

\begin{algorithm}[tb]
  \DontPrintSemicolon
  \caption{Steganography Encode} \label{alg: steganography-encode}
  \KwInput{A text $x$ in a Unicode sequence; A secret message $m \in \mathbb{Z}_{\ge 0}$; A list of codepoints $[u_0, \dots, u_{p-1}]$.}
  \KwOutput{A text that is embedded the secret message in a Unicode sequence.}
  \BlankLine
  $m_1, \dots, m_k \gets$ the $p$-ary form of $m$ \tcp*{$m_i \in \{0, 1, \ldots, p-1\}$}
  $i \gets 1$ \tcp*{The index of $m$}
  \For{$c$ \textup{\textbf{in}} $x$}{
      \If{$c$ \textup{is a whitespace}}{
          $c \gets u_{m_i}$ \tcp*{Replace a whitespace with a codepoint in $C$}
          $i \gets i + 1$ \tcp*{Increment the index}
          \If{$i > k$}{
              \textbf{break} \tcp*{Finish if the secret message is embedded}
          }
      }
  }
  \If{$i \le k$}{
      \textbf{raise} \texttt{ValueError} \tcp*{Raise an error if the secret message is too long}
  }
  \Return{$x$}
\end{algorithm}

\begin{algorithm}[tb]
  \DontPrintSemicolon
  \caption{Steganography Decode} \label{alg: steganography-decode}
  \KwInput{A text $x$ that is embedded a secret message in a Unicode sequence; A list of codepoints $[u_0, \dots, u_{p-1}]$.}
  \KwOutput{The secret message $m \in \mathbb{Z}_{\ge 0}$.}
  \BlankLine
  $m \gets 0$ \tcp*{Initialize the secret message}
  \For{$c$ \textup{\textbf{in}} $x$}{
    \For{$i = 0, \ldots, p-1$}{
      \If{$c = u_i$}{
          $m \gets p \times m + i$ \tcp*{Update the secret message}
      }
    }
  }
  \Return{$m$}
\end{algorithm}

\textsc{Whitemark} can be used as a steganography method. A user can embed a secret message in a text by choosing the places of whitespaces to be replaced. For example, if the secret message is \texttt{1101001} in binary, the first, second, fourth, and seventh whitespaces are replaced with U+2004. It should be noted that secret messages can be encrypted by standard encrypted methods such as AES and then fed to steganography so that the secret message is not read by adversaries. Secret messages can also be encoded by an error correction code and then fed to steganography to make the steganography more robust to edit. For example, suppose the message a user wants to embed is \begin{align}
  \texttt{x = 10100100001}.
\end{align} Encode it with an error correction code and suppose the encoded message is \begin{align}
  \texttt{x' = 001001000100001}.
\end{align} Replace the whitespaces with U+2004 based on the non-zero indices of \texttt{x'}. Even if someone edits the text and the whitespace pattern is changed to \begin{align}
  \texttt{x'' = 001001010100001},
\end{align} the validator can confirm \begin{align}
  \texttt{x = 10100100001}
\end{align} by decoding \texttt{x''} with the error correction code. The steganography technique also makes \textsc{Whitemark} more robust as a watermarking method. Each user can choose a different pattern and detect the watermark with more reliability.

This idea can easily be extended to other codepoints. Figure \ref{fig: whitespaces} lists codepoints for whitespaces. Each user can select a different set of codepoints for a watermark, and embed a secret message in a $p$-ary form, where $p$ is the number of codepoints in the set. Algorithms \ref{alg: steganography-encode} and \ref{alg: steganography-decode} show the pseudo-code of the steganography.

\subsection{Watermarks for CJK Texts (Variantmark)}

\begin{figure}[tb]
  \centering
  \includegraphics[width=\hsize]{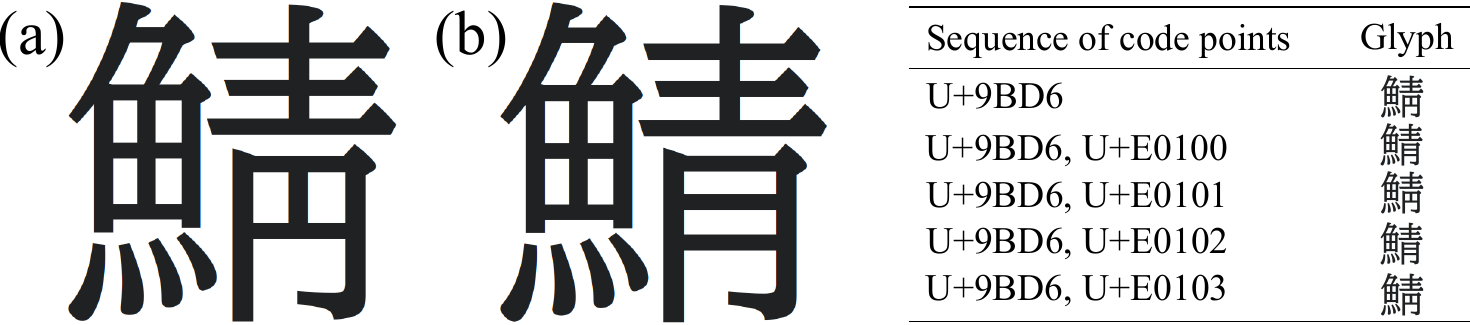}
  \caption{Variant characters and variation selectors. Both (a) and (b) have the meaning of mackerel. Although these characters are exchangeable in most scenarios, Unicode supports distinguishing them by variation selectors for some special purposes. The right table shows the list of Unicode sequences for the characters. These variations can be used for watermarking.}
  \label{fig: saba}
\end{figure}

\begin{figure}[tb]
  \centering
  \includegraphics[width=\hsize]{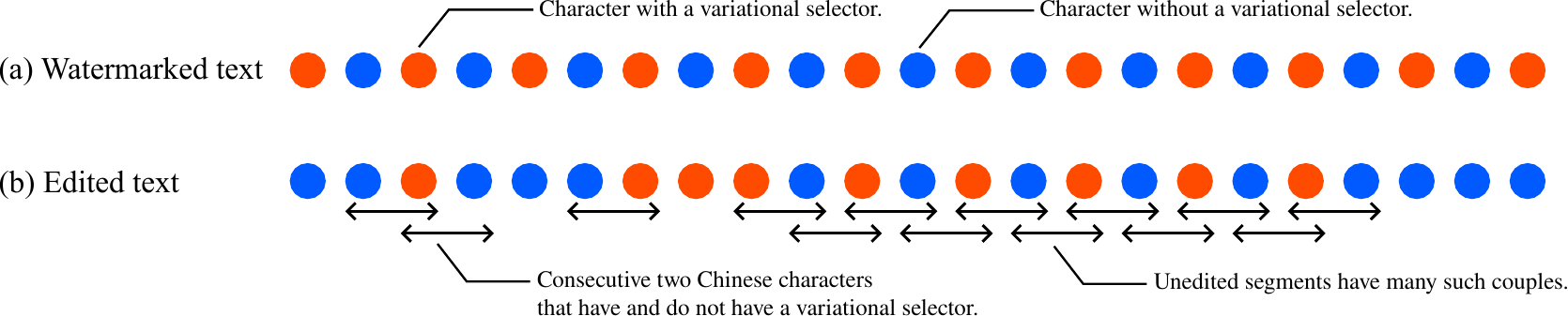}
  \caption{An illustration of \textsc{Variantmark}. (a) A watermarked text has alternating patterns. A red circle indicates a character with a variational character, and a blue circle indicates a character without a variational character. (b) The patterns remain even if the text is edited to some extent. }
  \label{fig: variantmark}
\end{figure}

\begin{figure}[tb]
  \centering
  \includegraphics[width=0.8\hsize]{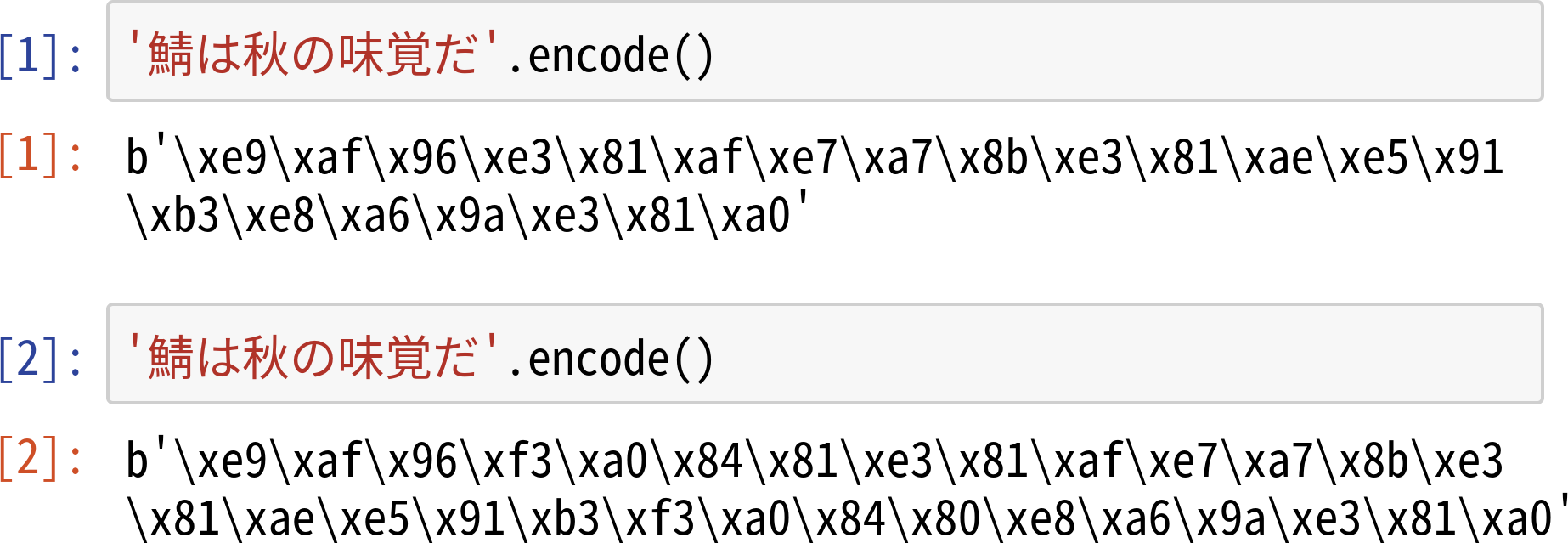}
  \caption{A screenshot of \textsc{Variantmark} for Japanese texts. Although the original text and watermarked text look the same, the Unicode sequences are different.}
  \label{fig: saba_comparison}
\end{figure}

Although European languages including English entail many whitespaces as delimiters, CJK languages such as Chinese and Japanese have few whitespaces. Therefore, \textsc{Whitemark} cannot be directly applied to CJK texts. We propose \textsc{Variantmark} for CJK texts. The idea is to use the variation selectors of Unicode. Some Chinese characters have variations with the same meaning, and Unicode supports specifying the variation by special code points. Figure \ref{fig: saba} shows an example of such characters. Figure \ref{fig: saba} (a) is the basic character for mackerel and has the codepoint U+9BD6. Figure \ref{fig: saba} (b) is a variant character. Unicode supports distinguishing these characters by variation selectors. The right table shows the list of Unicode sequences for the characters. These variations can be used for watermarking. A user can embed a secret message by choosing the variation selectors. The important fact is that Figure \ref{fig: saba} (a) can also be specified by variation selectors, as well as the single code point. Specifically, Figure \ref{fig: saba} (a) can be represented by [U+9BD6, U+E0101] and [U+9BD6, U+E0103] as well as [U+9BD6]. Similarly, other Chinese characters can be represented in at least two ways keeping their appearances. These choices are different as Unicode sequences, but their appearances are the same. \textsc{Variantmark} replaces U+9BD6 with [U+9BD6, U+E0101], and the validator can detect the watermark by checking the variation selectors. More specifically, \textsc{Variantmark} replaces every other occurrence of a Chinese character with one with a variation selector. The validator can detect the watermark by counting the number of consecutive two Chinese characters that have and do not have a variational selector (Figure \ref{fig: variantmark}). Note that we do not replace all of the Chinese characters but every other occurrence because natural texts have some variational selectors, and false positive could happen if \textsc{Variantmark} adopted the same strategy as \textsc{Whitemark}. Alternating patterns do not appear naturally, so the watermarks can be detected robustly. \textsc{Variantmark} can also be used for steganography. As natural texts have some variational selectors, secret messages should be encoded with an error correction code in the preprocessing.

Figure \ref{fig: saba_comparison} shows an example of execution. The original text and watermarked text look the same, but the Unicode sequences are different. The validator can detect the watermark by checking the variation selectors.

\subsection{Watermarks for Printed Texts (Printmark)}

\begin{figure}[tb]
  \centering
  \includegraphics[width=0.4\hsize]{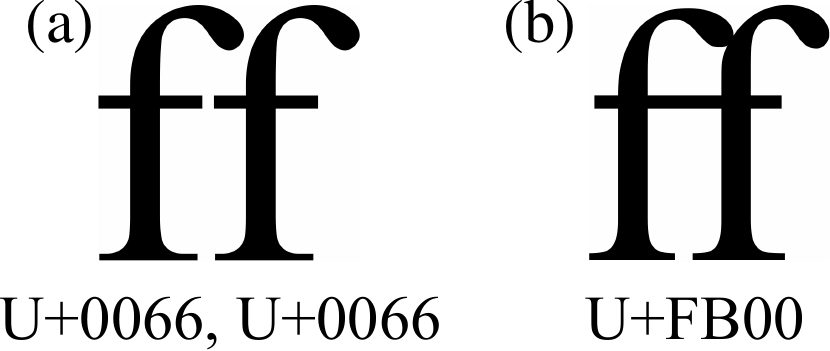}
  \caption{An example of ligature. (a) ff without ligature (b) ff with ligature. We can specify ligature with Unicode.}
  \label{fig: ligatures}
\end{figure}

\begin{figure}[tb]
  \centering
  \includegraphics[width=0.7\hsize]{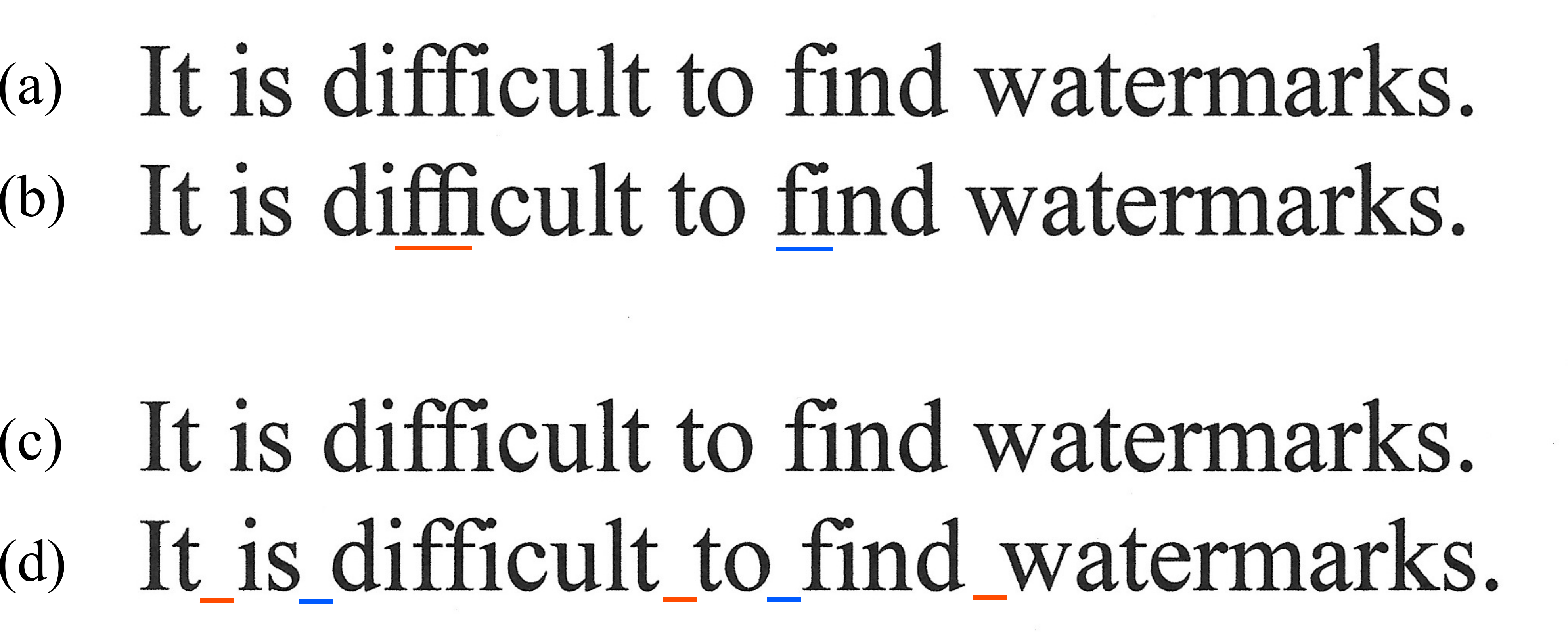}
  \caption{An example of execution of \textsc{Printmark}. We printed the texts out and scanned it to create the above image. (a) The original text. (b) The watermarked text with ligature. The red underline indicates ligature and the blue underline indicates non-ligature. (c) The original text. (d) The watermarked text with three-per-em spaces. The red underlines indicate space (U+0020), and the blue underlines indicate three-per-em spaces (U+2004).}
  \label{fig: print}
\end{figure}

One of the limitations of \textsc{Whitemark} is that it disappears when a texts is printed because the absolute lengths of U+2004 depend on the font and texts with U+2004 look normal once printed. Therefore, if an essay is printed, the teacher cannot detect \textsc{Whitemark}.

We propose several methods, which we call \textsc{Printmark}, to cope with printed texts. The first idea is to use ligature (Figure \ref{fig: ligatures}). \textsc{Printmark} replaces every other occurrence of a substring that can be expressed as ligature with that with ligature. We do not replace all of the substrings but every other occurrence because natural texts have some ligatures, and the false positive could happen if \textsc{Printmark} adopted the same strategy as \textsc{Whitemark}. The second idea is to use whitespaces with slightly different lengths. \textsc{Printmark} replaces every other occurrence of a whitespace with a three-per-em space (U+2004). Although this slightly changes the appearance, the changes are hardly perceptible. A validator can detect the watermark by checking the pattern of whitespaces. The third idea is to use variant characters of Chinese characters. For example, \textsc{Printmark} replaces U+9BD6 with [U+9BD6, U+E0100] (i.e., Figure \ref{fig: saba} (a) to (b)). Although this slightly changes the appearance, the meaning of a text does not change, and most users will not notice anything. Figure \ref{fig: print} shows examples of \textsc{Printmark}. The appearances are hardly changed, but they are indeed changed, and the validator can detect the watermark.

\subsection{Limitation} \label{sec: limitation}

The critical limitation of \textsc{Whitemark} is that it can be bypassed by replacing all whitespaces with the basic whitespace U+0020, then the validator can no longer detect the watermark. We argue that this limitation does not undermine the value of \textsc{Whitemark}.

First, most of end users are not familiar with the specifications of Unicode. For example, high school students do not know the meaning of different code points and hardly notice nor cope with the difference between U+0020 and U+2004. The does not need to achive the perfect recall. Once some student is caught by a teacher for writing an essay using an LLM, other students would refrain from taking the risk of using LLMs for their essays.

Second, the false negative is a universal problem of watermarking, as we will formally show in the next subsection. \textit{Any} watermarking method has this drawback. Therefore, the criticism that the watermark can be erased is not specific to \textsc{Easymark} but is valid for all watermarking methods, including elaborated ones. As this is inevitable in principle, practitioners will all the more hesitate to adopt complicated methods, and \textsc{Easymark} is a practical compromise solution. 

\subsection{Impossibility Theorem} \label{sec: impossibility}

We show that it is impossible to construct a perfect watermark.

\begin{theorem}[Impossibility Theorem, Informal]
  There exists a universal erasing function that erases any reliable watermark without much degradation of the quality of the text.
\end{theorem}

The theorem is formally stated as follows.

\begin{theorem}[Impossibility Theorem, Formal] \label{thm: impossibility}
  Let $(\mathcal{X}, d_\mathcal{X})$ be a metric space of texts. Let $C$ be the random variable that indicates a condition (i.e., prompt). Let $X = f(C)$ be the text generated by an LLM given the condition $C$. Let $g\colon \mathcal{C} \times \mathcal{K} \to \mathcal{X}$ be any watermarking function and $X_k = g(C, k)$ be the text with a watermark with key or random seed $k$, where $\mathcal{C}$ is the space of conditions, and $\mathcal{K}$ is the space of keys. Suppose \begin{align}
    \mathbb{E}[\mathcal{L}(X_k, C) - \mathcal{L}(X, C)] \le \varepsilon \label{eq: alt-thm-loss1}
  \end{align} holds, where $\mathcal{L}\colon \mathcal{X} \times \mathcal{C} \to \mathbb{R}$ is a loss function that is $1$-Lipschitz continuous with respect to the first argument, and $\varepsilon$ is a positive number, i.e., the quality of the text is not degraded much with the watermark. Let $\textup{\texttt{Detect}}\colon \mathcal{X} \times \mathcal{K} \to \{\textup{\texttt{True}}, \textup{\texttt{False}}\}$ be any function such that \begin{align}
    \textup{Pr}[\textup{\texttt{Detect}}(X_k, k) = \textup{\texttt{True}}] &\ge 1 - \delta,  \\
    \textup{Pr}[\textup{\texttt{Detect}}(X, k) = \textup{\texttt{False}}] &= 1, \label{eq: alt-thm-detect2}
  \end{align} i.e., the watermark can be detected reliably. Suppose \begin{align}
    \mathbb{E}[d_\mathcal{X}(X, X_k)] &\le \varepsilon', \label{eq: alt-thm-d}
  \end{align} hold, i.e., the watermark does not change the meaning of the text much. Then, there exists $\textup{\texttt{Erase}}\colon \mathcal{X} \to \mathcal{X}$ such that \begin{align}
    \textup{Pr}[\textup{\texttt{Detect}}(\textup{\texttt{Erase}}(X_k), k) = \textup{\texttt{False}}] &= 1, \label{eq: alt-thm-erase} \\
    \mathbb{E}[\mathcal{L}(\textup{\texttt{Erase}}(X_k), C) - \mathcal{L}(X, C)] &\le \varepsilon + \varepsilon', \label{eq: alt-thm-erase-loss1}
  \end{align} i.e., the watermark can be erased without harming the quality of the text and without knowing the key $k$, and $\textup{\texttt{Erase}}$ is universal in the sense that it does not depend on $g$, $k$, \textup{\texttt{Detect}}, or prompts but only on $X$.
\end{theorem}

The proof is available in Appendix \ref{sec: proof}. 

In practice, \texttt{Erase}(x) can be approximately simulated by translating $x$ into French and back to English by DeepL or Google Translate. This does not change the meaning of the text but erases the watermark.

We assume that the watermark does not change the meaning of the text much in Theorem \ref{thm: impossibility}, i.e., Eq. \ref{eq: alt-thm-d}. This is the case for most watermarking methods, including those proposed by \cite{kirchenbauer2023watermark} and \cite{takezawa2023necessary} and \textsc{Easymark}. Note that Eq. \ref{eq: alt-thm-d} is automatically met if the loss $\mathcal{L}(\cdot, c)$ is unimodal and both $X$ and $X_k$ incur low losses because $X$ and $X_k$ are in the same basin and close to each other in this case. We show that this assumption is necessary for the theorem to hold in the following by showing a counterexample.

\textbf{Counterexample.} We show that Theorem \ref{thm: impossibility} does not hold if a watermark does not meet Eq. \ref{eq: alt-thm-d}, i.e., it does not care about the metric. Let $\mathcal{X} = \{x_1, x_2, x_3\}, \mathcal{C} = \{c_1, c_2\}$, and let $C$ follow the uniform distribution on $\mathcal{C}$. Let the loss function be \begin{align}
  \mathcal{L}(x_1, c_1) &= 0, & \mathcal{L}(x_2, c_1) &= \infty, & \mathcal{L}(x_3, c_1) &= 0, \\
  \mathcal{L}(x_1, c_2) &= \infty, & \mathcal{L}(x_2, c_2) &= 0, & \mathcal{L}(x_3, c_2) &= 0.
\end{align} Let the generated texts be \begin{align}
  f(c_1) &= x_1, & f(c_2) &= x_2, & g(c_1, k) &= x_3, & g(c_2, k) &= x_3.
\end{align} Let the detection function be \begin{align}
  \texttt{Detect}(x_1, k) &= \texttt{False}, & \texttt{Detect}(x_2, k) &= \texttt{False}, & \texttt{Detect}(x_3, k) &= \texttt{True}.
\end{align} With the above conditions, \begin{align}
  \mathcal{L}(X, C) &= 0, \\
  \mathcal{L}(X_k, C) &= 0, \\
  \texttt{Detect}(X, k) &= \texttt{False}, \text{and} \\
  \texttt{Detect}(X_k, k) &= \texttt{True}
\end{align} hold with probability one, i.e., the watermark is perfect. However, there is not a good erasing function $\texttt{Erase}\colon \mathcal{X} \to \mathcal{X}$ for this watermark. If $\texttt{Erase}(x_3) = x_1$, then under $C = c_2$, \begin{align}
  \mathcal{L}(\texttt{Erase}(X_k), C) &= \mathcal{L}(\texttt{Erase}(x_3), c_2) \\
  &= \mathcal{L}(x_1, c_2) \\
  &= \infty,
\end{align} and if $\texttt{Erase}(x_3) = x_2$, then under $C = c_1$, \begin{align}
  \mathcal{L}(\texttt{Erase}(X_k), C) &= \mathcal{L}(\texttt{Erase}(x_3), c_1) \\
  &= \mathcal{L}(x_2, c_1) \\
  &= \infty,
\end{align} and if $\texttt{Erase}(x_3) = x_3$, then \begin{align}
  \texttt{Detect}(\texttt{Erase}(X_k), k) &= \texttt{Detect}(\texttt{Erase}(x_3), k) \\
  &= \texttt{Detect}(x_3, k) \\
  &= \texttt{True}.
\end{align}

\textbf{Remark (Implications)} Theorem \ref{thm: impossibility} is interesting in its own right because Theorem \ref{thm: impossibility}, its assumptions, and the above counterexample tell the promising directions of watermarking methods. As Theorem \ref{thm: impossibility} shows, it is impossible to design a perfect watermark as long as the watermark does not change the meaning of the text. Therefore, if one aims at designing unbreakable watermarks, the watermark should be designed to change the meaning of the text while it should be a reasonable answer to the prompt, and one should focus on tasks the loss of which is multimodal, like the above counterexample. For example, if there are two reasonable but different answers $x_1$ and $x_2$ for a prompt $c$, and the LLM generates $x_1$ for $c$, then an unbreakable watermark should output $x_2$ instead of similar texts to $x_1$. This is not the case for tasks such as summarization and translation because all correct answers to a prompt are usually similar to each other. Exploring tasks where multimodality exists and watermarking methods that exploit the multimodality are promising directions. Essay writing is one such a direction because there are many reasonable answers for a single condition $c$, and watermarking is indeed important for essay writing. Alternatively, one can also exploit the assumption of Theorem \ref{thm: impossibility} that the detection function should recognize the text without a watermark almost surely, which is necessary for the theorem to hold. We discuss it in detail in Appendix \ref{sec: ext-theorem}. This indicates that watermarks that are difficult to erase can be designed if we make a trade-off between the detection accuracy and the erasing difficulty. As stated above, Theorem \ref{thm: impossibility} reveals directions that are doomed to failure and provides guidance on fruitful directions.

In summary, Theorem \ref{thm: impossibility} is valuable in three ways. First, it motivates us to use an easy watermarking method like \textsc{Easymark}. Second, it tells us not to rely too much on watermarking methods. Practitioners should be aware of the existence of watermark erasers and should not be overconfident about the certainty of watermarking however sophisticated the method is. Third, it provides guidance on fruitful directions to design theoretically sound watermarks.

\section{Experiments} \label{sec: experiment}

\begin{table}[t]
  \centering
  \caption{BLEU scores and detection accuracy with NLLB-200-3.3B and WMT.}
  \label{table: translation}
  \begin{tabular}{lcccccc}
  \toprule
   & \multicolumn{3}{c}{En $\rightarrow$ De} & \multicolumn{3}{c}{De $\rightarrow$ En}  \\
   & BLEU $\uparrow$ & FNR $\downarrow$ & FPR $\downarrow$ & BLEU $\uparrow$ & FNR $\downarrow$ & FPR $\downarrow$ \\
  \midrule
  w/o Watermark             & $\mathbf{36.4}$ & n.a. & n.a. & $\mathbf{42.6}$ & n.a. & n.a. \\
  Soft-Watermark \citep{kirchenbauer2023watermark}           &  $5.2$ & $3.0 \%$ & $0.4 \%$ & $7.5$ & $3.3 \%$ & $0.5 \%$ \\
  Adaptive Soft-Watermark  & $20.5$ & $\mathbf{0.0} \%$ & $2.6 \%$ & $20.6$ & $\mathbf{0.0} \%$ & $1.9 \%$ \\
  NS-Watermark \citep{takezawa2023necessary}             & $32.7$ & $\mathbf{0.0} \%$ & $0.3 \%$ & $38.2$ & $\mathbf{0.0} \%$ & $\mathbf{0.0} \%$ \\
  \textsc{Whitemark} (Ours) & $\mathbf{36.4}$ & $0.1 \%$ & $\mathbf{0.0} \%$ & $\mathbf{42.6}$ & $\mathbf{0.0} \%$ & $\mathbf{0.0} \%$ \\
  \bottomrule
  & & & & \\
  \toprule
   & \multicolumn{3}{c}{En $\rightarrow$ Fr} & \multicolumn{3}{c}{Fr $\rightarrow$ En}  \\
   & BLEU $\uparrow$ & FNR $\downarrow$ & FPR $\downarrow$ & BLEU $\uparrow$ & FNR $\downarrow$ & FPR $\downarrow$ \\
  \midrule
  w/o Watermark             & $\mathbf{42.6}$ & n.a. & n.a. & $\mathbf{40.8}$ & n.a. & n.a. \\
  Soft-Watermark \citep{kirchenbauer2023watermark}           &  $9.6$ & $5.4 \%$ & $0.3 \%$ & $7.6$ & $3.6 \%$ & $0.6 \%$ \\
  Adaptive Soft-Watermark  & $23.3$ & $\mathbf{0.0} \%$ & $2.2 \%$ & $19.5$ & $\mathbf{0.0} \%$ & $2.8 \%$ \\
  NS-Watermark \citep{takezawa2023necessary}             & $38.8$ & $\mathbf{0.0} \%$ & $0.3 \%$ & $36.8$ & $\mathbf{0.0} \%$ & $0.1 \%$ \\
  \textsc{Whitemark} (Ours) & $\mathbf{42.6}$ & $\mathbf{0.0} \%$ & $\mathbf{0.0} \%$ & $\mathbf{40.8}$ & $\mathbf{0.0} \%$ & $\mathbf{0.0} \%$ \\
  \bottomrule
  \end{tabular}
\end{table}

\begin{table}[t]
  \centering
  \caption{Text quality and detection accuracy with LLaMA-7B and C4 dataset.}
  \label{table: ppl_vs_accuracy}
  \begin{tabular}{lccc}
  \toprule
  & PPL $\downarrow$ & FNR $\downarrow$ & FPR $\downarrow$ \\
  \midrule
  w/o Watermark              & $\mathbf{1.85}$ & n.a. & n.a. \\
  Soft-Watermark  \citep{kirchenbauer2023watermark}           & $6.25$ & $2.8 \%$ & $0.1 \%$ \\
  Adaptive Soft-Watermark   & $2.48$ & $0.2 \%$ & $0.8 \%$ \\
  NS-Watermark \citep{takezawa2023necessary}              & $1.92$ & $\mathbf{0.0} \%$ & $0.3 \%$ \\
  \textsc{Whitemark} (Ours)  & $\mathbf{1.85}$ & $\mathbf{0.0} \%$ & $\mathbf{0.0} \%$ \\
  \bottomrule
  \end{tabular}
\end{table}

We confirm the effectiveness of our proposed method with two tasks and two LLMs. We compared our method with the following watermarking methods with the same hyperparameter settings used by \cite {takezawa2023necessary}:

\begin{itemize}
  \item The Soft-Watermark \citep[Algorithm 2]{kirchenbauer2023watermark} adds biases to the logits of specific words and detects the watermark by checking the ratio of the biased words.
  \item NS-Watermark \citep{takezawa2023necessary} follows the same idea but precisely controls the false positive ratio by dynamic programming.
  \item Adaptive Soft-Watermark \citep{takezawa2023necessary} is a variant of Soft-Watermark that controls the false positive ratio by binary search for each input.
\end{itemize}

The first task is machine translation. We used NLLB-200-3.3B \citep{costajussa2022no} as the language model and used the test dataset of WMT'14 French (Fr) $\leftrightarrow$ English (En) \citep{bojar2014findings} and WMT'16 German (De) $\leftrightarrow$ English (En) \citep{bojar2016findings}. We report the BLEU scores of the translation results by the NLLB model and the texts watermarked by the above and our methods. We also report the false negative ratio (FNR) and the false positive ratio (FPR) of the detection function. The FNR is the ratio of the texts with watermarks that are not detected, and the FPR is the ratio of the texts without watermarks that are detected. Table \ref{table: translation} shows that \textsc{Whitemark} consistently performs better than the other methods. \textsc{Whitemark} outperforms the state-of-the-art watermarking method NS-Watermark with $10$ percent relative improvements of BLEU and is more reliable in terms of detection accuracy. We emphasize again that \textsc{Whitemark} is much easier to implement and deploy than NS-Watermark. \textsc{Whitemark} is perfect in the sense that the BLEU scores of the watermarked texts are the same as those of the original texts, and the FNR and FPR are almost zero. There are only four false negative examples, each of which contains only one word, in which case we do not need to claim that the text was generated by an LLM. This indicates that \textsc{Whitemark} can add a watermark without harming the quality of the text and can detect the watermark reliably. These results are consistent with Propositions \ref{prop: bleu} and \ref{prop: detect}.

The second task is text completion. We used LLaMA-7B \citep{touvron2023llama} as the language model and used the subsets of C4, realnews-like dataset \citep{raffel2020exploring}. We followed the experimental setups used by \cite{kirchenbauer2023watermark} and \cite{takezawa2023necessary}. Specifically, we split each text into 90 percent and 10 percent lengths and input the first into the language model. We computed the perplexity (PPL) of the generated text by the LLaMA model and texts watermarked by the above and our methods. We computed the perplexity by an auxiliary language model, which is regarded as an oracle to measure the quality of the text. We also report the FNR and FPR of the detection function. Table \ref{table: ppl_vs_accuracy} shows that \textsc{Whitemark} consistently performs better than the other methods in terms of both PPL and detection accuracy.

These results are impressive because even \textsc{Whitemark} can achieve the almost perfect watermarking performance. These results indicate that future watermarking methods should not focus only on BLEU, PPL, FNR, and FPR, and use these metrics only for sanity checks.

\section{Related Work}

With the recent progress of LLMs, the demand for detecting whether a text is generated by an LLM is increasing. There are two main ways of doing this. The first approach is blackbox detection \citep{gehrmann2019gltr,uchendu2020authorship,gambin2022pushing}, which does not require intervention in the model. These methods exploit the statistical tendency of texts generated by LLMs \citep{mitchell2023detectgpt,guo2023how}. However, as the LLMs become more sophisticated, the statistical tendency becomes less obvious, and the blackbox detection becomes less reliable \citep{clark2021all,jakesch2023human,schuster2020limitatinos}. Another approach is whitebox detection, which requires intervention in the model, including inference-time watermarking \citep{kirchenbauer2023watermark,takezawa2023necessary,christ2023undetectable} and post-hoc watermarking as \cite{he2022protecting} and \cite{venugopal2011watermarking} and ours, retrieval-based detection \citep{krishna2023paraphrasing}, and linguistic steganography \citep{fang2017generating,dai2019towards,ziegler2019neural,ueoka2021frustratingly}. Watermarking methods are sometimes used for detecting model extraction attacks \citep{he2022cater, peng2023are, zhao2023protecting,gu2022watermarking}. Although watermarking methods are more reliable than blackbox detection \citep{kirchenbauer2023reliability}, the main drawback is that it harms the quality of the text. Our method is reliable while it does not harm the quality of the text, as Propositions \ref{prop: bleu} and \ref{prop: detect} and the experiments show.

\textsc{Easymark} can be implemented as a user-side system \citep{sato2022private, sato2022retrieving}. \cite{sato2022clear} pointed out that ``Even if a user of the service is unsatisfied with a search engine and is eager to enjoy additional functionalities, what he/she can do is limited. In many cases, he/she continues to use the unsatisfactory system or leaves the service.'' and proposed a user-side realization method to solve this problem. The spirit of \textsc{Easymark} is the same. Even if the official LLM provider does not offer watermarking, users can use \textsc{Easymark} to add a watermark to a text. This is a practical advantage of \textsc{Easymark}.

The idea of using Unicode encoding for watermarking is not new \citep{por2012unispach,rizzo2016content,rizzo2017text}. The differences between our work and the existing Unicode-based watermarks are two-fold. First, the existing works are not in the context of LLMs, and tackle different problems. \textsc{Easymark} is designed as simple as possible so that it can be easily implemented and deployed with LLMs. Second, we provided theoretical justifications in Propositions \ref{prop: bleu} and \ref{prop: detect} and Theorem \ref{thm: impossibility}. These results are valuable in their own right as we discussed in Section \ref{sec: impossibility}. It would be an interesting future direction to extend our theoretical results to contexts other than LLMs, including those tackled in the existing Unicode-based watermarks.

Finally, \cite{sadasivan2023can} also showed the impossibility of a perfect watermark. The results of \citep{sadasivan2023can} also justify the use of simple watermarks like \textsc{Easymark}. The difference between the theory shown in \cite{sadasivan2023can} and ours is that we assume general loss functions and elucidate in which cases the watermark can be erased more precisely. Our positive and negative results can be used for designing watermarks as discussed in Section \ref{sec: impossibility}.

\section{Conclusion}

We proposed \textsc{Easymark}, a family of embarrassingly easy watermarking methods. \textsc{Easymark} is simple and easy to implement and deploy. Nevertheless, \textsc{Easymark} has preferable theoretical properties that ensure the quality of the text and the reliability of the watermark. The simplicity and the theoretical properties make \textsc{Easymark} attractive for practitioners. We also proved that it is impossible to construct a perfect watermark and any watermark can be erased. This result is valuable in its own right because it motivates us to use an easy watermarking method like \textsc{Easymark}, encourages us not to rely too much on watermarking methods, and provides guidance on fruitful directions to design theoretically sound watermarks. We confirmed the effectiveness of \textsc{Easymark} with the experiments involving LLM generated texts. \textsc{Easymark} outperforms the state-of-the-art watermarking methods in terms of BLEU and perplexity and is more reliable in terms of detection accuracy. We encourage practitioners to use \textsc{Easymark} as a starter for watermarking methods and recommend LLM researchers use \textsc{Easymark} as a simple yet strong baseline.



\subsubsection*{Acknowledgments}
Yuki Takezawa, Ryoma Sato, and Makoto Yamada were supported by JSPS KAKENHI Grant Number 23KJ1336, 21J22490, and MEXT KAKENHI Grant Number 20H04243, respectively.

\bibliography{main}
\bibliographystyle{tmlr}

\appendix
\section{Proof of Theorem \ref{thm: impossibility}} \label{sec: proof}

\begin{proof}
  Let \begin{align}
    \mathcal{S} \stackrel{\text{def}}{=} \{x \in \mathcal{X} \mid \text{Pr}[X = x] > 0\} \subset \mathcal{X}
  \end{align} be the support of $X$, i.e., the set of possible texts generated by $f$. Let $\texttt{Erase}\colon \mathcal{X} \to \mathcal{X}$ be \begin{align}
    \texttt{Erase}(x) \stackrel{\text{def}}{=} \argmin_{x' \in \mathcal{S}} d_\mathcal{X}(x, x').
  \end{align} Ties can be broken arbitrarily. We show that \texttt{Erase} satisfies the conditions. Take any watermark $g$, and key $k$. From Eq. \ref{eq: alt-thm-detect2}, for any $x \in \mathcal{S}$, $\texttt{Erase}(x) = \texttt{False}$ holds. As $\texttt{Erase}(x) \in \mathcal{S}$, $\texttt{Detect}(\texttt{Erase}(x)) = \texttt{False}$ holds surely. Therefore, Eq. \ref{eq: alt-thm-erase} holds. Besides, \begin{align}
    &\mathbb{E}[\mathcal{L}(\textup{\texttt{Erase}}(X_k), C) - \mathcal{L}(X, C)]\\
    &= \mathbb{E}[\mathcal{L}(\textup{\texttt{Erase}}(X_k), C) - \mathcal{L}(X_k, C)] + \mathbb{E}[\mathcal{L}(X_k, C) - \mathcal{L}(X, C)] \\
    &\stackrel{(a)}{\le} \mathbb{E}[d_\mathcal{X}(\textup{\texttt{Erase}}(X_k), X_k)] + \mathbb{E}[\mathcal{L}(X_k, C) - \mathcal{L}(X, C)] \\
    &\stackrel{(b)}{\le} \mathbb{E}[d_\mathcal{X}(\textup{\texttt{Erase}}(X_k), X_k)] + \varepsilon \\
    &\stackrel{(c)}{\le} \mathbb{E}[d_\mathcal{X}(X, X_k)] + \varepsilon \\
    &\stackrel{(d)}{\le} \varepsilon + \varepsilon'
  \end{align} hold, where (a) is due to the Lipschitzness of $\mathcal{L}$, (b) is due to Eq. \ref{eq: alt-thm-loss1}, (c) is due to the definition of \texttt{Erase}, and (d) is due to Eq. \ref{eq: alt-thm-d}. Therefore, Eq. \ref{eq: alt-thm-erase-loss1} holds, and \texttt{Erase} satisfies the conditions.
\end{proof}

\section{Extention of Theorem \ref{thm: impossibility}} \label{sec: ext-theorem}

We assume that the text without a watermark can be detected almost surely, i.e., \begin{align}
  \textup{Pr}[\textup{\texttt{Detect}}(X, k) = \textup{\texttt{False}}] &= 1
\end{align} in Theorem \ref{thm: impossibility}. We relax this assumption in the following. Let $q(x)$ be the probability mass function of $X_k$. We show that the assumption can be loosened if we allow \texttt{Erase} to depend on $q$. The same \texttt{Erase} can be used for different watermarks $g$ and $k$ as long as $q$ is the same. $q$ can be observed and estimated by a user by drawing samples from watermarked texts. Therefore, dependence on $q$ is a reasonable assumption.

\begin{theorem} \label{thm: impossibility-ext}
  Let $(\mathcal{X}, d_\mathcal{X})$ be a metric space of texts. Let $C$ be the random variable that indicates a condition (i.e., prompt). Let $X = f(C)$ be the text generated by an LLM given the condition $C$. Let $g\colon \mathcal{C} \times \mathcal{K} \to \mathcal{X}$ be any watermarking function and $X_k = g(C, k)$ be the text with a watermark with key $k$, where $\mathcal{C}$ is the space of conditions, and $\mathcal{K}$ is the space of keys. Let $q(x)$ be the probability mass function of $X_k$. Suppose \begin{align}
    \mathbb{E}[\mathcal{L}(X_k, C) - \mathcal{L}(X, C)] \le \varepsilon \label{eq: ext-thm-loss1}
  \end{align} holds, where $\mathcal{L}\colon \mathcal{X} \times \mathcal{C} \to \mathbb{R}$ is a loss function that is $1$-Lipschitz continuous with respect to the first argument, and $\varepsilon$ is a positive number, i.e., the quality of the text is not degraded much with the watermark. Let $\textup{\texttt{Detect}}\colon \mathcal{X} \times \mathcal{K} \to \{\textup{\texttt{True}}, \textup{\texttt{False}}\}$ be any function such that \begin{align}
    \textup{Pr}[\textup{\texttt{Detect}}(X_k, k) = \textup{\texttt{True}}] &\ge 1 - \delta,  \\
    \textup{Pr}[\textup{\texttt{Detect}}(X, k) = \textup{\texttt{False}}] &\ge 1 - \delta, \label{eq: ext-thm-detect2}
  \end{align} i.e., the watermark can be detected reliably. Suppose \begin{align}
    \mathbb{E}[d_\mathcal{X}(X, X_k)] &\le \varepsilon', \label{eq: ext-thm-d}
  \end{align} hold, i.e., the watermark does not change the meaning of the text much. Then, there exists a randomized function $\textup{\texttt{Erase}}\colon \mathcal{X} \to \mathcal{X}$ such that \begin{align}
    \textup{Pr}[\textup{\texttt{Detect}}(\textup{\texttt{Erase}}(X_k), k) = \textup{\texttt{False}}] &\ge 1 - \delta, \label{eq: ext-thm-erase} \\
    \mathbb{E}[\mathcal{L}(\textup{\texttt{Erase}}(X_k), C) - \mathcal{L}(X, C)] &\le \varepsilon + \varepsilon', \label{eq: ext-thm-erase-loss1}
  \end{align} i.e., the watermark can be erased without harming the quality of the text and without knowing the key $k$, and $\textup{\texttt{Erase}}$ is universal in the sense that it does not depend on $g$, $k$, \textup{\texttt{Detect}}, or propmts but only on $X$ and $q$.
\end{theorem}

\begin{proof}
  Let $q$ be any probability mass function on $\mathcal{X}$ such that there exist $\tilde{g}$ and $\tilde{k}$ such that $\tilde{X}_k = \tilde{g}(C, \tilde{k})$ follows $q$ and \begin{align}
    \mathbb{E}[d_\mathcal{X}(X, \tilde{X}_k)] &\le \varepsilon'. \label{eq: ext-thm-dtilde}
  \end{align} Take any such $\tilde{g}$ and $\tilde{k}$. Let \begin{align}
    \text{Pr}[\texttt{Erase}_q(x) = x'] \stackrel{\text{def}}{=} \text{Pr}[X = x' \mid \tilde{X}_k = x]. \label{eq: ext-thm-erase-def}
  \end{align} We show that \texttt{Erase} satisfies the conditions. Take any watermark $g$ and key $k$ such that $X_k = g(C, k)$ follows $q$ and satisfy the assumptions of Theorem \ref{thm: impossibility-ext}. From Eq. \ref{eq: ext-thm-erase-def}, $\texttt{Erase}_q(X_k)$ follows the same distribution as $X$, and therefore, \begin{align}
    \textup{Pr}[\textup{\texttt{Detect}}(\texttt{Erase}_q(X_k), k) = \textup{\texttt{False}}] &\ge 1 - \delta \\
  \end{align} holds due to Eq. \ref{eq: ext-thm-detect2}. Besides, \begin{align}
    &\mathbb{E}[\mathcal{L}(\texttt{Erase}_q(X_k), C) - \mathcal{L}(X, C)]\\
    &= \mathbb{E}[\mathcal{L}(\texttt{Erase}_q(X_k), C) - \mathcal{L}(X_k, C)] + \mathbb{E}[\mathcal{L}(X_k, C) - \mathcal{L}(X, C)] \\
    &\stackrel{(a)}{\le} \mathbb{E}[d_\mathcal{X}(\texttt{Erase}_q(X_k), X_k)] + \mathbb{E}[\mathcal{L}(X_k, C) - \mathcal{L}(X, C)] \\
    &\stackrel{(b)}{\le} \mathbb{E}[d_\mathcal{X}(\texttt{Erase}_q(X_k), X_k)] + \varepsilon \\
    &= \left(\sum_{x' \in \mathcal{X}} \sum_{x_k \in \mathcal{X}} d_\mathcal{X}(x', x_k) \text{Pr}[\texttt{Erase}(x_k) = x'] q(x_k) \right) + \varepsilon \\
    &\stackrel{(c)}{=} \left( \sum_{x' \in \mathcal{X}} \sum_{x_k \in \mathcal{X}} d_\mathcal{X}(x', x_k) \text{Pr}[X = x' \mid \tilde{X}_k = x_k] q(x_k) \right) + \varepsilon \\
    &\stackrel{(d)}{=} \mathbb{E}[d_\mathcal{X}(X, \tilde{X}_k)] + \varepsilon \\
    &\stackrel{(e)}{\le} \varepsilon + \varepsilon',
  \end{align} hold, where (a) is due to the Lipschitzness of $\mathcal{L}$, (b) is due to Eq. \ref{eq: ext-thm-loss1}, (c) is due to the definition of $\texttt{Erase}_q$, i.e., Eq. \ref{eq: ext-thm-erase-def}, (d) follows the fact that $X_k$ also follows $q$, and (e) is due to Eq. \ref{eq: ext-thm-dtilde}. Therefore, Eq. \ref{eq: ext-thm-erase-loss1} holds, and \texttt{Erase} satisfies the conditions.
\end{proof}

The dependence on $q$ is necessary. We show that there are no universal erasing functions that satisfy the conditions without depending on $q$ by a counterexample.

\textbf{Counterexample.} Let $\mathcal{X} = \{0, 1, 2, \ldots, n\}$, $\mathcal{C} = \{1, 2, \ldots, n\}$, and $C$ follow the uniform distribution on $\mathcal{C}$. Let the language model be $f(c) = c \in \{1, 2, \ldots, n\}$. Let the loss function be zero everywhere. Take any erasing function $\texttt{Erase}\colon \mathcal{X} \to \mathcal{X}$. We show that there exists an adversarial watermark $g, k$ such that \texttt{Erase} cannot erase the watermark.

Case 1 ($\texttt{Erase}(0) = 0$): Let $g(c, k) = 0$ and \begin{align}
  \texttt{Detect}(x, k) &\stackrel{\text{def}}{=} \begin{cases}
    \texttt{True} & (x = 0) \\
    \texttt{False} & (\text{otherwise})
  \end{cases},
\end{align} then \begin{align}
  &\text{Pr}[\texttt{Detect}(X, k) = \texttt{True}] = \text{Pr}[X = 0] = 0, \\
  &\text{Pr}[\texttt{Detect}(X_k, k) = \texttt{True}] = \text{Pr}[X_k = 0] = 1, \\
  &\text{Pr}[\texttt{Detect}(\texttt{Erase}(X_k), k) = \texttt{True}] = \text{Pr}[\texttt{Detect}(0) = \texttt{True}] = 1,
\end{align} i.e., \texttt{Erase} fails to erase the watermark.

Case 2 ($\texttt{Erase}(0) = i^* \neq 0$): Let $g(c, k) = 0$ and \begin{align}
  \texttt{Detect}(x) &\stackrel{\text{def}}{=} \begin{cases}
    \texttt{True} & (x \in \{0, i^*\}) \\
    \texttt{False} & (\text{otherwise})
  \end{cases},
\end{align} then \begin{align}
  &\text{Pr}[\texttt{Detect}(X, k) = \texttt{True}] = \text{Pr}[X \in \{0, i^*\}] = \frac{1}{n}, \\
  &\text{Pr}[\texttt{Detect}(X_k, k) = \texttt{True}] = \text{Pr}[X_k \in \{0, i^*\}] = 1, \\
  &\text{Pr}[\texttt{Detect}(\texttt{Erase}(X_k), k) = \texttt{True}] = \text{Pr}[\texttt{Detect}(i^*) = \texttt{True}] = 1,
\end{align} i.e., \texttt{Erase} fails to erase the watermark.

\end{document}